\documentclass{article}

\usepackage[preprint]{ewrl_2026}

\usepackage[utf8]{inputenc}
\usepackage[T1]{fontenc}
\usepackage[colorlinks=true, citecolor=blue]{hyperref}
\usepackage{url}
\usepackage{booktabs}
\usepackage{amsfonts}
\usepackage{amsmath}
\usepackage{amssymb}
\usepackage{amsthm}
\usepackage{nicefrac}
\usepackage{microtype}
\usepackage{xcolor}
\usepackage{algorithm}
\usepackage{algorithmic}
\usepackage{mathtools}

\newtheorem{theorem}{Theorem}
\newtheorem{lemma}{Lemma}
\newtheorem{proposition}{Proposition}
\newtheorem{corollary}{Corollary}
\newtheorem{definition}{Definition}
\newtheorem{remark}{Remark}
\newtheorem{assumption}{Assumption}
\newtheorem{example}{Example}

\DeclareMathOperator*{\argmin}{arg\,min}
\DeclareMathOperator{\supp}{supp}
\newcommand{\alg}{\texttt{Fair-ETC-TPZSG}}
\newcommand{\E}{\mathbb{E}}
\newcommand{\Prob}{\mathbb{P}}
\newcommand{\R}{\mathbb{R}}
\newcommand{\Sx}{\mathcal{S}_x}
\newcommand{\Sy}{\mathcal{S}_y}
\newcommand{\Sact}{\mathcal{S}_{\mathrm{act}}}
\newcommand{\Sflr}{\mathcal{S}_{\mathrm{flr}}}
\newcommand{\F}{\mathcal{F}}
\newcommand{\ind}{\mathbf{1}}
\newcommand{\eps}{\varepsilon}
\newcommand{\fair}{\mathrm{fair}}
\newcommand{\tA}{\widetilde{A}}
\newcommand{\tp}{\widetilde{p}}
\newcommand{\tq}{\widetilde{q}}
\newcommand{\tDelta}{\widetilde{\Delta}}

\title{Fairness in two-player zero-sum games with\\ bandit feedback}

\author{S Akash \\ LatentForce.ai \\ Bengaluru, India
        \And
        Pratik Gajane \\ Laboratoire d'Informatique Fondamentale d'Orléans \\ University of Orléans, France
}

\begin{document}
\maketitle

\begin{abstract}
We study two-player zero-sum games (TPZSGs) with bandit feedback under fairness constraints requiring every action to be played with probability at least $\alpha/m$.
Existing instance-dependent results target $\textit{pure}$ Nash equilibria, while fairness generically produces $\textit{mixed}$ equilibria, a harder learning target.
Our key technical tool is a reparametrization: every fair strategy decomposes as $p = (\alpha/m)\mathbf{1} + (1-\alpha)\widetilde{p}$ with $\widetilde{p} \in \Delta_m$, and substituting into the payoff form yields $p^{\top}Aq = \widetilde{p}^{\top}\widetilde{A} q$ for a fair payoff matrix $\widetilde{A} := (1-\alpha)A + \alpha\mathbf{1} c^{\top}$, where $c_j = \tfrac{1}{m}\sum_i A(i,j)$ is the column-mean vector.
The fair game on $A$ is then equivalent to a standard zero-sum game on $\widetilde{A}$, so equilibrium existence, KKT structure, and LP basis stability reduce to classical results applied to $\widetilde{A}$.
We derive the fair minimax value, fair Nash equilibrium, fair regret, and a clean dual representation showing the price of fairness is at most $\alpha(1-1/m)$ and vanishes whenever the unconstrained equilibrium already has full support.
Our main result is an $\widetilde{O}(T^{2/3})$ regret bound for an Explore-Then-Commit algorithm, $\texttt{Fair-ETC-TPZSG}$, applicable to general mixed fair equilibria, together with a discussion of why naive action elimination does not readily improve it.
When the fair equilibrium has a single dominant action, equivalently when $\widetilde{p}^{\star}$ is a vertex of $\Delta_m$, the bound sharpens to instance-dependent $\widetilde{O}(1/\widetilde{\Delta}(\alpha)^{2})$, where $\widetilde{\Delta}(\alpha)$ is the LP-margin gap.
\end{abstract}

\section{Introduction}
\label{sec:intro}

Two-player zero-sum games (TPZSGs) are a foundational object of study in game theory and online learning \citep{vonneumann1928,nash1951,lattimore2020}.
When the payoff matrix is unknown and estimated through bandit feedback, the learner must balance exploration against exploitation while contending with an adversarial opponent.
Recent progress has been substantial: \citet{ito2025} established instance-dependent regret bounds for self-play of the Tsallis-INF algorithm of \citet{zimmert2021}, achieving an $O(c_{1}\log T + \sqrt{c_{2}T})$ bound that accelerates the worst-case $O(\sqrt T)$ rate, while \citet{yilmaz2025} analyzed Explore-Then-Commit (ETC) and obtained $O(\log T/\Delta)$ rates whenever the game admits a pure Nash equilibrium (NE) with gap $\Delta$.

A common limitation of this line of work, acknowledged by \citet{yilmaz2025}, is the restriction to pure NE.
We show that this restriction becomes \emph{critical} once fairness constraints enter the picture.
Consider a setting where the row player's actions correspond to serving distinct groups, and the column player is an adversary shaping the environment.
Such settings are common: a system allocating resources under adversarial demand, a recommendation system facing strategic manipulation, or a content moderator acting across demographic categories.
In all of them the optimal unconstrained strategy is free to concentrate its mass on a few high-value actions.
This is acceptable when actions are interchangeable, but when they correspond to groups it means some groups are served rarely or not at all.
A natural remedy is the \emph{minimum-play constraint}: every action receives probability at least $\alpha/m$, where $\alpha \in [0,1]$ controls fairness stringency.
This constraint generically transforms the equilibrium from pure to mixed, \emph{even when the unconstrained game admits a pure NE}.

The shift from pure to mixed equilibria is not merely cosmetic.
For pure NE, correct identification of the optimal action yields zero commit-phase regret in ETC-type algorithms, enabling $O(\log T)$ rates.
For mixed NE, the committed strategy depends continuously on the estimated payoff matrix and is never exactly equal to its target, producing $O(\eps)$ per-round error that accumulates to $\Theta(T\eps)$ over the horizon.
Balancing this against $\Theta(\tau m\ell)$ exploration cost yields an $O(T^{2/3})$ rate, a barrier we believe is intrinsic to ETC schemes; an adaptive algorithm that re-solves the LP on improving estimates would be needed to do better.

\paragraph{Technical approach.}
The analysis centers on a reparametrization (Section~\ref{sec:reparam}) that expresses any fair strategy $p \in \F_{\alpha}$ as $p = (\alpha/m)\ind + (1-\alpha)\tp$ for $\tp \in \Delta_m$, yielding $p^{\top}\!Aq = \tp^{\top}\tA q$ for the fair payoff matrix $\tA := (1-\alpha)A + \alpha\,\ind\, c^{\top}$ with $c_j = (1/m)\sum_i A(i,j)$.
The fair game on $A$ is then equivalent to a standard zero-sum game on $\tA$, and every property we need (equilibrium existence, the dual representation, KKT structure, LP basis stability) follows from classical results applied to $\tA$.

\paragraph{Contributions.}
\begin{enumerate}
\item We introduce the \emph{fair TPZSG} model and the reparametrization reduction, deriving the fair minimax value, fair Nash equilibrium, fair regret, and LP-derived suboptimality gaps (Section~\ref{sec:setup}).
A clean dual representation expresses the fair value as a convex combination of minimax and uniform-strategy values, and the price of fairness is bounded by $\alpha(1-1/m)$ in three lines (Section~\ref{sec:pof}).

\item We propose \alg\ and prove an $O(T^{2/3})$ regret bound for general (possibly mixed) fair NE (Theorem~\ref{thm:fetc_general}), and discuss why LP structure obstructs naive action elimination as a route to improving it (Section~\ref{sec:fetc}).

\item Under a single-dominant-action assumption (Assumption~\ref{ass:dominant}: $\tp^{\star}$ is a vertex of $\Delta_m$) and LP non-degeneracy, \alg\ attains an instance-dependent $O(\log T)$ rate that scales with the LP-margin gap $\tDelta(\alpha)$ and basis conditioning $\kappa(\alpha)$ (Theorem~\ref{thm:fetc_inst}).
At $\alpha = 0$ (pure NE) this recovers an instance-dependent $O(\log T)$ rate as a special case.
\end{enumerate}

\section{Problem setup}
\label{sec:setup}

\subsection{Two-player zero-sum games with bandit feedback}

A TPZSG is specified by a row-action set $\Sx = \{1,\ldots,m\}$, a column-action set $\Sy = \{1,\ldots,\ell\}$, and a payoff matrix $A \in \R^{m\times\ell}$ \citep{yilmaz2025}.
Over $T$ rounds, at round $t$ the row player selects $i_{t}\in\Sx$ and the column player selects $j_{t}\in\Sy$; both observe each other's actions, and the row player observes a noisy payoff $r_{t}$ with $\E[r_{t}\mid i_{t},j_{t}] = A(i_{t},j_{t})$.

\begin{assumption}[Bounded payoffs and sub-Gaussian noise]
\label{ass:setup}
$A(i,j) \in [0,1]$ for all $(i,j)$, and noise is bounded so that $r_{t} - A(i_{t},j_{t})$ is conditionally $\sigma$-sub-Gaussian given $(i_{t},j_{t})$ with $\sigma \le 1/2$ (Hoeffding for $[0,1]$-bounded rewards).
\end{assumption}

The standard unconstrained minimax value is $V^{\star} := \max_{p\in\Delta_{m}}\min_{q\in\Delta_{\ell}} p^{\top}\!A q$.
We write $\hat A_{t}(i,j) = n_{ij,t}^{-1}\sum_{s\le t}\ind\{(i_{s},j_{s})=(i,j)\}\,r_{s}$ for the empirical payoff matrix after $t$ rounds, with $n_{ij,t}$ the play count of pair $(i,j)$.

\subsection{Fairness constraints}

\begin{definition}[Fair strategy set]
\label{def:fair}
For $\alpha\in[0,1]$, the \emph{fair strategy set} is
\[
\F_{\alpha} = \bigl\{p\in\Delta_{m} : p_{i}\ge \alpha/m\;\;\forall\,i\in\Sx\bigr\}.
\]
\end{definition}

At $\alpha=0$, $\F_{0}=\Delta_{m}$; at $\alpha=1$, $\F_{1}=\{\ind/m\}$.
For $\alpha\in(0,1)$, $\F_{\alpha}$ is a non-trivial polytope.
This minimum-play notion is natural when row actions index distinct user groups and no group may be wholly neglected; it parallels the meritocratic fairness of \citet{joseph2016} and the fair-quota model of \citet{patil2021}.

\begin{remark}[One-sided vs. two-sided fairness]
\label{rem:twosided}
Our model constrains only the row player.
If both players are fair-constrained ($q_{j} \ge \alpha/\ell$ as well, with set $\mathcal G_{\alpha}$ for the column), the column player's set shrinks, so $V^{\star}_{\F,\mathcal G} := \max_{p\in\F_{\alpha}}\min_{q\in\mathcal G_{\alpha}}p^{\top}\!A q \ge V^{\star}_{\F}$: the row player gains value.
However, the learning problem need not be uniformly easier under two-sided fairness, because the adversary still adapts to the committed strategy within $\mathcal G_{\alpha}$, and any analysis must control its perturbation sensitivity within the constrained polytope.

We treat the one-sided case throughout, but the reparametrization does carry over.
With a column floor $q_{j}\ge\beta/\ell$ and decomposition $q=\tfrac{\beta}{\ell}\ind+(1-\beta)\tq$ for $\tq\in\Delta_{\ell}$, expanding $p^{\top}\!Aq$ under both decompositions gives
\begin{equation}
\label{eq:twosided}
p^{\top}\!Aq = \tp^{\top}\tA_{2}\,\tq
\quad\text{for all } \tp\in\Delta_{m},\ \tq\in\Delta_{\ell},
\end{equation}
where, writing $d_{i}=\tfrac{1}{\ell}\sum_{j}A(i,j)$ for the row-mean vector and $\bar A=\tfrac{1}{m\ell}\sum_{i,j}A(i,j)$ for the grand mean,
\begin{equation}
\label{eq:tildeA2}
\tA_{2} = (1-\alpha)(1-\beta)\,A
        + (1-\beta)\,\alpha\,\ind c^{\top}
        + (1-\alpha)\,\beta\,d\,\ind^{\top}
        + \alpha\beta\,\bar A\,\ind\ind^{\top}.
\end{equation}
Two-sided fairness thus reduces to a standard game on $\tA_{2}$, which now carries \emph{two} rank-one corrections (one per floor) plus a constant term, rather than a single symmetric one; setting $\beta=0$ recovers $\tA_{2}=\tA$ from Definition~\ref{def:tildeA}.
Each entry $\tA_{2}(i,j)$ is a convex combination of $A(i,j)$, $c_{j}$, $d_{i}$, and $\bar A$ (the four coefficients are non-negative and sum to $1$), all of which lie in $[0,1]$; hence $\tA_{2}\in[0,1]^{m\times\ell}$ as before.
\end{remark}

\subsection{Reparametrization: the fair payoff matrix}
\label{sec:reparam}

The technical observation underpinning everything that follows is that $\F_{\alpha}$ is an affine image of the simplex, which lets us reduce the fair game to a standard zero-sum game on a modified payoff matrix.

\begin{lemma}[Affine parametrization of $\F_{\alpha}$]
\label{lem:reparam}
For $\alpha\in[0,1)$, every $p\in\F_{\alpha}$ admits a unique decomposition
\begin{equation}
\label{eq:reparam}
p = \tfrac{\alpha}{m}\ind + (1-\alpha)\tp, \qquad \tp\in\Delta_{m},
\end{equation}
with $\tp_{i} = (p_{i} - \alpha/m)/(1-\alpha)$.
\end{lemma}

\begin{proof}
Given $p\in\F_{\alpha}$, the stated $\tp$ satisfies $\tp_{i}\ge 0$ (since $p_{i}\ge\alpha/m$) and $\sum_{i}\tp_{i} = (1-\alpha)/(1-\alpha) = 1$.
Uniqueness is immediate from the affine bijection.
The case $\alpha=1$ is degenerate: $\F_{1} = \{\ind/m\}$ collapses to a point and $\tp$ is unidentifiable; we handle it separately when it arises.
\end{proof}

\begin{definition}[Fair payoff matrix]
\label{def:tildeA}
For $\alpha\in[0,1]$, the \emph{fair payoff matrix} $\tA\in\R^{m\times\ell}$ is
\begin{equation}
\label{eq:tildeA}
\tA(i,j) := (1-\alpha)A(i,j) + \alpha\,c_{j},\qquad c_{j} := \tfrac{1}{m}\textstyle\sum_{i'} A(i',j).
\end{equation}
Equivalently, $\tA = (1-\alpha)A + \alpha\,\ind\, c^{\top}$, a rank-one perturbation that pulls each column toward its mean.
Since $\tA$ is a convex combination of values in $[0,1]$, $\tA(i,j)\in[0,1]$.
\end{definition}

\begin{proposition}[Reduction to a standard game]
\label{prop:reduction}
For all $p\in\F_{\alpha}$ and $q\in\Delta_{\ell}$, the corresponding $\tp\in\Delta_{m}$ from \eqref{eq:reparam} satisfies
\begin{equation}
\label{eq:reduction}
p^{\top}\! A q = \tp^{\top}\tA\, q.
\end{equation}
Consequently, the fair TPZSG on $A$ is equivalent to the standard TPZSG on $\tA$ played with strategies $(\tp,q)\in\Delta_{m}\times\Delta_{\ell}$.
\end{proposition}

\begin{proof}
Substituting \eqref{eq:reparam} and using $\sum_{i}\tp_{i}=1$,
\begin{align*}
p^{\top}\! A q
&= \tfrac{\alpha}{m}\ind^{\top}\! A q + (1-\alpha)\tp^{\top}\! A q
= \alpha\textstyle\sum_{j}q_{j}c_{j} + (1-\alpha)\tp^{\top}\! A q\\
&= \textstyle\sum_{i,j}\tp_{i}q_{j}\bigl[(1-\alpha)A(i,j)+\alpha c_{j}\bigr]
= \tp^{\top}\tA\, q. \qedhere
\end{align*}
\end{proof}

Proposition~\ref{prop:reduction} is the workhorse: every property of the fair game on $A$ becomes a property of the standard game on $\tA$, with the simplex constraint on $\tp$ replacing the polytope constraint on $p$.
The fairness floor has been absorbed into the payoffs.

\subsection{Fair value, fair Nash equilibrium, and the dual representation}
\label{sec:value}

\begin{definition}[Fair minimax value]
\label{def:fairvalue}
\(
V^{\star}_{\F} := \max_{p\in\F_{\alpha}}\min_{q\in\Delta_{\ell}} p^{\top}\!A q = \max_{\tp\in\Delta_{m}}\min_{q\in\Delta_{\ell}} \tp^{\top}\tA\, q.
\)
\end{definition}

A \emph{fair Nash equilibrium} is a pair $(p^{\star}_{\F},q^{\star}_{\F})\in\F_{\alpha}\times\Delta_{\ell}$ achieving the value; equivalently, $(\tp^{\star},q^{\star}_{\F})$ is a Nash equilibrium of the standard game on $\tA$.
Existence follows from von Neumann's minimax theorem applied to $\tA$.

\begin{proposition}[Dual representation]
\label{prop:dual}
For any $\alpha\in[0,1]$,
\begin{equation}
\label{eq:dual_clean}
V^{\star}_{\F} = \min_{q\in\Delta_{\ell}}\Bigl[(1-\alpha)\max_{i\in\Sx}(Aq)_{i} + \tfrac{\alpha}{m}\textstyle\sum_{i}(Aq)_{i}\Bigr].
\end{equation}
\end{proposition}

\begin{proof}
Applying von Neumann's minimax theorem to $\tA$ on $\Delta_{m}\times\Delta_{\ell}$ and using that the maximum of a linear functional over $\Delta_{m}$ is attained at a vertex,
\(
V^{\star}_{\F} = \min_{q}\max_{\tp\in\Delta_{m}}\tp^{\top}\tA q = \min_{q}\max_{i}(\tA q)_{i}.
\)
Substituting $(\tA q)_{i} = (1-\alpha)(Aq)_{i} + (\alpha/m)\sum_{i'}(Aq)_{i'}$, the second summand is constant in $i$, so the maximum passes through to give \eqref{eq:dual_clean}.
\end{proof}

At $\alpha=0$ we recover $V^\star$; at $\alpha=1$ the row is forced to play uniform and $V^{\star}_{\F} = \min_q (1/m)\sum_i (Aq)_i$.

\subsection{KKT structure and LP gaps}
\label{sec:gaps}

Because the fair game is equivalent to the standard game on $\tA$, equilibrium structure is governed by the standard Nash KKT conditions, with no additional Lagrange multipliers for floor constraints.
At a fair NE $(\tp^{\star},q^{\star}_{\F})$ with value $\nu := V^{\star}_{\F}$:
\begin{align}
\label{eq:kkt_row_new}
\textstyle\sum_{j}q^{\star}_{\F}(j)\tA(i,j) &\le \nu \quad\forall\,i, \;\text{with equality if } \tp^{\star}_{i}>0,\\
\label{eq:kkt_col_new}
\textstyle\sum_{i}\tp^{\star}_{i}\tA(i,j) &\ge \nu \quad\forall\,j, \;\text{with equality if } q^{\star}_{\F}(j)>0.
\end{align}

\begin{definition}[Active/floor sets, row and column gaps]
\label{def:gaps}
Define $\Sact := \{i:\tp^{\star}_{i}>0\} = \{i:p^{\star}_{\F}(i)>\alpha/m\}$, $\Sflr := \Sx\setminus\Sact$, and
\[
\rho_{i}(\alpha) := \nu - \textstyle\sum_{j}q^{\star}_{\F}(j)\tA(i,j) \ge 0,\quad
\gamma_{j}(\alpha) := \textstyle\sum_{i}\tp^{\star}_{i}\tA(i,j) - \nu \ge 0.
\]
By \eqref{eq:kkt_row_new}--\eqref{eq:kkt_col_new}, $\rho_{i}=0$ for $i\in\Sact$ and $\gamma_{j}=0$ on $\supp(q^{\star}_{\F})$.
\end{definition}

\begin{definition}[Primal margin and basis conditioning]
\label{def:primal}
Write the standard NE LP on $\tA$ in primal form: maximize $\nu$ over $(\tp,\nu)\in\Delta_m\times\R$ subject to $\sum_i \tp_i \tA(i,j) \ge \nu$ for all $j\in\Sy$.
At an optimal vertex the binding constraints are the simplex equality, the floors $\tp_i = 0$ for $i\in\Sflr$, and the tight columns $\supp(q^{\star}_{\F})$.
Under LP non-degeneracy with strict complementarity this active set is unique and pins down the optimum through a square, invertible linear system in the basic variables $(\tp^{\star}_{\Sact},\nu)$, with $s := |\Sact|$; we write $M_{\mathcal B}\in\R^{(s+1)\times(s+1)}$ for its coefficient matrix (constructed explicitly in Appendix~\ref{app:basis}).
Two quantities measure how robustly the optimum is determined:
\[
\pi(\alpha) := \min_{i\in\Sact}\tp^{\star}_{i} > 0, \qquad \kappa(\alpha) := \|M_{\mathcal B}^{-1}\|_{\infty} < \infty,
\]
the smallest active mass and the conditioning of the basis. Both are well-defined under the non-degeneracy hypothesis.
\end{definition}

\begin{definition}[LP sensitivity gap]
\label{def:lpgap}
\[
\tDelta(\alpha) := \min\!\Bigl(\,\min_{j:\gamma_{j}>0}\gamma_{j},\;\min_{i:\rho_{i}>0}\rho_{i},\;\pi(\alpha)\,\Bigr),
\qquad\text{with the convention } \min\emptyset = +\infty.
\]
The first two terms capture dual-feasibility margins; the third captures the primal-feasibility margin.
The basis conditioning $\kappa(\alpha)$ enters the threshold via Lemma~\ref{lem:basis} below, where the perturbation budget is $\eps < \tDelta(\alpha)/(C\kappa(\alpha))$ for an explicit constant $C$.
\end{definition}

The dual-gap terms re-express in terms of $A$ via $\tA(i,j)=(1-\alpha)A(i,j)+\alpha c_j$: for any $i'\in\Sact$,
\[
\rho_{i}(\alpha) = (1-\alpha)\bigl[(Aq^{\star}_{\F})_{i'} - (Aq^{\star}_{\F})_{i}\bigr],
\]
so $\rho_{i}$ is the $(1-\alpha)$-scaled advantage of active actions over $i$ against $q^{\star}_{\F}$; $\gamma_j$ admits a similar column-shortfall reading.

\subsection{Fair regret}

\begin{definition}[Fair regret]
\label{def:regret}
\(
R_{T}^{\fair} := T\cdot V^{\star}_{\F} - \E\bigl[\textstyle\sum_{t=1}^{T} A(i_{t},j_{t})\bigr].
\)
\end{definition}

This measures regret against the best \emph{fair} strategy.

\section{Price of fairness}
\label{sec:pof}

\begin{proposition}[Price of fairness]
\label{prop:pof}
For any $A\in[0,1]^{m\times\ell}$ and $\alpha\in[0,1]$,
\[
0 \le V^{\star} - V^{\star}_{\F} \le \alpha(1 - 1/m).
\]
Moreover, if $p^{\star}_{i}\ge\alpha/m$ for all $i$ at some unconstrained optimum $p^{\star}$, then $V^{\star}_{\F}=V^{\star}$.
\end{proposition}

\begin{proof}
$V^{\star}_{\F}\le V^{\star}$ follows from $\F_{\alpha}\subseteq\Delta_{m}$.
For the upper bound, fix $q\in\Delta_{\ell}$ and let $M(q):=\max_{i}(Aq)_{i}$.
Since each entry $(Aq)_{i}\in[0,1]$ and at least one entry equals $M(q)$,
\(
(1/m)\sum_{i}(Aq)_{i} \ge (1/m)[M(q) + 0(m-1)] = M(q)/m,
\)
so $M(q) - (1/m)\sum_{i}(Aq)_{i} \le M(q)(1-1/m) \le 1-1/m$.
Substituting into \eqref{eq:dual_clean},
\[
V^{\star}_{\F} = \min_{q}\bigl[M(q) - \alpha(M(q) - \tfrac{1}{m}\textstyle\sum_{i}(Aq)_{i})\bigr] \ge \min_{q}[M(q) - \alpha(1-1/m)] = V^{\star} - \alpha(1-1/m).
\]
For the second claim, $p^{\star}\in\F_{\alpha}$ gives $V^{\star}_{\F}\ge V^{\star}$; combined with the upper bound, equality holds.
\end{proof}

The bound $\alpha(1-1/m)$ is attained, e.g., when one row pays $1$ and every other row pays $0$ regardless of the column, as in $A = \bigl(\begin{smallmatrix}1&1\\0&0\end{smallmatrix}\bigr)$: here $V^{\star}=1$ but $V^{\star}_{\F}=1-\alpha/2$, so $V^{\star}-V^{\star}_{\F}=\alpha(1-1/2)$ exactly. Intuitively, tightness requires the adversary to force the row player off its unconstrained optimum by the maximum possible amount. The bound is loose whenever the unconstrained optimum already lies in $\F_{\alpha}$.

\section{\alg: Explore-Then-Commit under fairness}
\label{sec:fetc}

\begin{algorithm}[t]
\caption{\alg$(m,\ell,T,\alpha,\tau)$}
\label{alg:fetc}
\begin{algorithmic}[1]
\REQUIRE Action sets $\Sx,\Sy$, horizon $T$, fairness level $\alpha\in[0,1]$, exploration length $\tau\in\mathbb Z_{+}$
\STATE \textbf{Explore phase (coordinated):} For each pair $(i,j)\in\Sx\times\Sy$ in any fixed order, schedule $\tau$ rounds in which the row plays $i$ and the column is instructed to play $j$.
\STATE Form $\hat A(i,j) = \tau^{-1}\sum_{s:\,(i_{s},j_{s})=(i,j)} r_{s}$ and $\hat\tA(i,j) := (1-\alpha)\hat A(i,j) + \alpha\hat c_{j}$ with $\hat c_{j} = (1/m)\sum_{i}\hat A(i,j)$.
\STATE \textbf{Commit phase:} Solve the standard NE LP on $\hat\tA$ to obtain $\hat\tp$; set $\hat p_{\F} := (\alpha/m)\ind + (1-\alpha)\hat\tp$.
\FOR{$t = \tau m\ell + 1,\ldots,T$}
    \STATE Sample $i_{t}\sim\hat p_{\F}$; column adversary best-responds on the true matrix $A$.
\ENDFOR
\end{algorithmic}
\end{algorithm}

By Proposition~\ref{prop:reduction}, solving the NE LP on $\hat\tA$ is computationally equivalent to solving the constrained LP on $\hat A$ with $p_i\ge\alpha/m$ (Appendix~\ref{app:lp_compat}); we phrase Algorithm~\ref{alg:fetc} via $\hat\tA$ because basis-perturbation theory is cleaner without floor constraints. Coordinated exploration in line~1 follows the explore-then-commit analysis of \citet{yilmaz2025} in the bandit-TPZSG setting, which inspired our scheme.

During exploration (line~1) the column player is \emph{coordinated}: it plays the scheduled action $j$ so that every pair $(i,j)$ is sampled exactly $\tau$ times. This is a strong protocol assumption, presuming a controlled or cooperative exploration phase.
During the commit phase the column player reverts to an \emph{omniscient adversary} that best-responds to the committed strategy $\hat p_{\F}$ on the true matrix $A$. The benchmark $V^{\star}_{\F}$ in the fair regret (Definition~\ref{def:regret}) is taken against this adversarial commit-phase behavior.

\subsection{General regret bound}

We first record a lemma showing that estimation error in $A$ is preserved (not amplified) when passing to $\tA$.

\begin{lemma}[Estimation error is preserved]
\label{lem:err}
If $\|\hat A - A\|_{\infty}\le\eps$, then $\|\hat\tA - \tA\|_{\infty}\le\eps$.
\end{lemma}

\begin{proof}
$|\hat c_{j}-c_{j}|\le(1/m)\sum_{i}|\hat A(i,j)-A(i,j)|\le\eps$, so $|\hat\tA(i,j)-\tA(i,j)|\le(1-\alpha)\eps + \alpha\eps = \eps$.
\end{proof}

\begin{theorem}[General regret bound for mixed fair NE]
\label{thm:fetc_general}
Under Assumption~\ref{ass:setup}, suppose the horizon satisfies the admissibility condition $T \ge 2\sigma^{2}m\ell\,L$ with $L = \log(2m\ell T)$, so that the exploration budget fits within the horizon ($\tau^{\star}m\ell\le T$).
Then \alg\ with $\tau = \lceil(T\sigma\sqrt{2L}/(m\ell))^{2/3}\rceil$ achieves
\begin{equation}
R_{T}^{\fair} \le 3(2\sigma^{2}m\ell T^{2}L)^{1/3} + m\ell + 1.
\end{equation}
\end{theorem}

\begin{proof}
\emph{Step 1 (Concentration).}
After exploration, each pair $(i,j)$ is sampled $\tau$ times.
Hoeffding's inequality under $\sigma$-sub-Gaussianity gives, for any $\eps>0$,
\(
\Prob[|\hat A(i,j) - A(i,j)|>\eps] \le 2\exp(-\tau\eps^{2}/(2\sigma^{2})).
\)
A union bound over $m\ell$ pairs yields $\Prob[\|\hat A-A\|_{\infty}>\eps]\le 1/T$ for
\begin{equation}
\label{eq:eps}
\eps = \sigma\sqrt{2L/\tau}, \qquad L := \log(2m\ell T).
\end{equation}
By Lemma~\ref{lem:err}, the same $\eps$ controls $\|\hat\tA-\tA\|_{\infty}$ on this event.

\emph{Step 2 (Fair-strategy transfer under the good event).}
Condition on $\mathcal E := \{\|\hat A - A\|_{\infty}\le\eps\}$, which occurs with probability at least $1-1/T$.
By Proposition~\ref{prop:reduction} applied to the random strategy $\hat p_{\F}$:
\begin{equation}
\label{eq:committed_value}
\min_{j}\hat p_{\F}^{\top}\!A e_{j} = \min_{j}\hat\tp^{\top}\tA\, e_{j},
\end{equation}
since the identity $p^{\top}\!Aq=\tp^{\top}\tA q$ holds for every $q$, including the column adversary's best response.
Now apply the min-max transfer argument on $\tA$:
\begin{align}
\min_{j}\hat\tp^{\top}\tA\, e_{j}
&\ge \min_{j}\hat\tp^{\top}\hat\tA\, e_{j} - \eps && \text{[Lemma~\ref{lem:err}]} \notag\\
&\ge \min_{j}(\tp^{\star})^{\top}\hat\tA\, e_{j} - \eps && \text{[$\hat\tp$ maximizes $\min_{j}\tp^{\top}\hat\tA e_{j}$ over $\Delta_{m}\ni\tp^{\star}$]} \notag\\
&\ge \min_{j}(\tp^{\star})^{\top}\tA\, e_{j} - 2\eps && \text{[Lemma~\ref{lem:err} again]} \notag\\
&= V^{\star}_{\F} - 2\eps. \label{eq:transfer_chain}
\end{align}
Combining \eqref{eq:committed_value}--\eqref{eq:transfer_chain}, on $\mathcal E$ the per-round commit-phase payoff is at least $V^{\star}_{\F} - 2\eps$.

\emph{Step 3 (Regret decomposition).}
The exploration phase contributes at most $\tau m\ell$ rounds of regret, each at most $1$.
On $\mathcal E$, the commit phase contributes at most $2\eps(T-\tau m\ell)\le 2T\eps$.
On $\mathcal E^{c}$, total regret is bounded by $T\cdot V^{\star}_{\F}\le T$, with $\Prob[\mathcal E^{c}]\le 1/T$, contributing at most $1$:
\begin{equation}
\label{eq:regret_raw}
R_{T}^{\fair} \le \tau m\ell + 2T\eps + 1.
\end{equation}

\emph{Step 4 (Choice of $\tau$).}
The optimization $\min_{\tau}\{m\ell\,\tau + 2T\sigma\sqrt{2L}\,\tau^{-1/2}\}$ is solved at $\tau^{\star}=(T\sigma\sqrt{2L}/(m\ell))^{2/3}$ with value $3(2\sigma^{2}m\ell T^{2}L)^{1/3}$ (Appendix~\ref{app:proof}); integer rounding adds at most $m\ell$.
\end{proof}

\begin{remark}[Why $T^{2/3}$ ?]
\label{rem:whyt23}
We believe this rate is intrinsic to ETC-style algorithms on mixed-fair-NE instances: because the committed strategy is fixed after exploration and cannot be refined as estimates improve, the commit-phase error is unavoidable.
\end{remark}

\begin{remark}[Theorem~\ref{thm:fetc_general} is independent of $\alpha$]
\label{rem:alpha_indep}
The bound of Theorem~\ref{thm:fetc_general} carries no dependence on $\alpha$: by Lemma~\ref{lem:err} the reparametrization $A\mapsto\tA$ is a convex combination and hence $1$-Lipschitz in $\ell_\infty$, so the estimation error driving the regret is unaffected by the fairness level. Fairness enters only through the benchmark $V^{\star}_{\F}$, which shifts downward by at most $\alpha(1-1/m)$ (Proposition~\ref{prop:pof}); the worst-case learning rate itself is fairness-robust.
\end{remark}

\subsection{Instance-dependent bound when $\tp^{\star}$ is a vertex of $\Delta_m$}

When the fair NE has special structure, namely $|\Sact|=1$, the commit-phase suboptimality vanishes entirely on the good event and the rate sharpens to $O(\log T)$.
This regime generalizes the pure-NE setting of \citet{yilmaz2025}, which it recovers at $\alpha = 0$ (Corollary~\ref{cor:pure_recovery}).

\begin{assumption}[Single dominant action]
\label{ass:dominant}
$|\Sact|=1$: there exists $i^{\star}\in\Sx$ with $\tp^{\star}=e_{i^{\star}}$, equivalently $p^{\star}_{\F}(i^{\star})=1-(m-1)\alpha/m$ and $p^{\star}_{\F}(i)=\alpha/m$ for $i\ne i^{\star}$.
\end{assumption}

\begin{example}[Non-degenerate single-dominant-action instance]
\label{ex:nondeg}
Let $m=3$, $\ell=2$, $\alpha=0.4$, and
\[
A = \begin{pmatrix} 1.0 & 0.6 \\ 0.2 & 0.4 \\ 0.3 & 0.5 \end{pmatrix}.
\]
Then $c=(0.5,0.5)^\top$ and $\tA = \begin{pmatrix} 0.80 & 0.56 \\ 0.32 & 0.44 \\ 0.38 & 0.50 \end{pmatrix}$.
Direct LP solution gives $\tp^{\star} = e_1$, $q^{\star}_{\F} = e_2$, $V^{\star}_{\F} = 0.56$, $p^{\star}_{\F} = (0.7\overline 3,\,0.1\overline 3,\,0.1\overline 3)$.
Gaps are strictly positive: $\rho_2 = 0.12$, $\rho_3 = 0.06$, $\gamma_1 = 0.24$, $\pi = 1$.
The basis is non-degenerate (unique optimal basis with strict complementarity), so Theorem~\ref{thm:fetc_inst} applies.
\end{example}

We need a basis-stability lemma to certify that the empirical LP recovers the same active set.
The proof requires care: under perturbation $\tA\to\hat\tA$, both the basic primal variables and the LP value shift, and the dual-feasibility margins (the gaps $\rho_i,\gamma_j$) shift by $O(\kappa(\alpha)\eps)$ rather than $O(\eps)$.

\begin{lemma}[LP basis stability]
\label{lem:basis}
Suppose the standard NE LP on $\tA$ has a unique non-degenerate optimal basis $\mathcal B$ with strict complementarity, and let $M_{\mathcal B}$, $\kappa(\alpha)$, $\pi(\alpha)$, $\tDelta(\alpha)$ be as in Definitions~\ref{def:primal}--\ref{def:lpgap}.
Write $s := |\Sact|$, set $C := 2s^{2}+2s+1$, and assume
\begin{equation}
\label{eq:lem3_hyp}
\|\hat\tA - \tA\|_{\infty} \le \eps, \qquad s\,\kappa(\alpha)\,\eps \le \tfrac{1}{2}, \qquad \eps \le \frac{\tDelta(\alpha)}{C\,\kappa(\alpha)}.
\end{equation}
Then the NE LP on $\hat\tA$ has the same optimal basis $\mathcal B$: $\hat\Sact=\Sact$, $\hat\Sflr=\Sflr$, and $\supp(\hat q_{\F})=\supp(q^{\star}_{\F})$.
Moreover, $\|\hat\tp^{\star}-\tp^{\star}\|_{\infty}\le 2s\kappa(\alpha)\eps$ and $|\hat\nu - \nu|\le 2s\kappa(\alpha)\eps$.
In the only application of this lemma (Theorem~\ref{thm:fetc_inst}), Assumption~\ref{ass:dominant} forces $s=1$, giving $C=5$ and the bounds $2\kappa(\alpha)\eps$.
\end{lemma}

\begin{proof}[Proof sketch]
Basis $\mathcal B$ stays optimal on $\hat\tA$ iff (i) the basic primal solution stays feasible and (ii) the dual reduced costs stay non-negative. A Neumann-series bound on the perturbed basis matrix $\hat M_{\mathcal B}$ gives $\|\hat\tp^{\star}-\tp^{\star}\|_{\infty},|\hat\nu-\nu|\le 2s\kappa(\alpha)\eps$, which preserves (i) since $2s\kappa(\alpha)\eps<\pi(\alpha)$; propagating these into the column and row slacks shows each shifts by at most $C\kappa(\alpha)\eps\le\tDelta(\alpha)$, preserving (ii). Full details, including the construction of $M_{\mathcal B}$ and the constant $C=2s^2+2s+1$, are in Appendix~\ref{app:basis}.
\end{proof}

\begin{remark}[On the two perturbation hypotheses in \eqref{eq:lem3_hyp}]
\label{rem:hyp}
The condition $s\kappa(\alpha)\eps\le 1/2$ ensures invertibility of the perturbed basis matrix via Neumann series; the condition $\eps\le\tDelta(\alpha)/(C\kappa(\alpha))$ ensures the perturbed dual-feasibility margins remain positive.
Since $\tDelta(\alpha)\le 1$ and $C=2s^{2}+2s+1\ge 2s$, the second condition gives $s\kappa(\alpha)\eps\le s\tDelta(\alpha)/C\le 1/2$ and so implies the first; in practice it suffices to enforce $\eps\le\tDelta(\alpha)/(C\kappa(\alpha))$ alone.
\end{remark}

\begin{theorem}[Instance-dependent bound]
\label{thm:fetc_inst}
Fix $\alpha\in[0,1)$ \emph{such that} Assumption~\ref{ass:dominant} and LP non-degeneracy hold at the fair NE on $\tA$ (an $A$-dependent condition; see Remark~\ref{rem:scope}). Under Assumptions~\ref{ass:setup}--\ref{ass:dominant}, \alg\ with
\[
\tau \;=\; \Bigl\lceil \frac{200\,\sigma^{2}\kappa(\alpha)^{2}\log(2m\ell T)}{\tDelta(\alpha)^{2}}\Bigr\rceil
\]
achieves
\begin{equation}
R_{T}^{\fair} \;\le\; \frac{200\,\sigma^{2}m\ell\,\kappa(\alpha)^{2}\log(2m\ell T)}{\tDelta(\alpha)^{2}} + m\ell + 1.
\end{equation}
\end{theorem}
\begin{proof}
Set $\eps := \tDelta(\alpha)/(C\kappa(\alpha))$, the threshold from Lemma~\ref{lem:basis}.
Step~1 of Theorem~\ref{thm:fetc_general} (with Lemma~\ref{lem:err}) guarantees $\Prob[\|\hat\tA-\tA\|_\infty>\eps]\le 1/T$ as long as $\tau\ge 8\sigma^{2}\log(2m\ell T)/\eps^{2}$. Under Assumption~\ref{ass:dominant} Lemma~\ref{lem:basis} has $s=1$, so $C=5$ and $\eps=\tDelta(\alpha)/(5\kappa(\alpha))$, giving $8\sigma^2\log(2m\ell T)/\eps^2 = 200\,\sigma^2\kappa(\alpha)^2\log(2m\ell T)/\tDelta(\alpha)^2$, which holds by the choice of $\tau$. On the good event $\mathcal E := \{\|\hat\tA-\tA\|_{\infty}\le\eps\}$, Lemma~\ref{lem:basis} certifies $\hat\Sact=\Sact=\{i^{\star}\}$.

Under Assumption~\ref{ass:dominant}, $\hat\Sact = \{i^{\star}\}$ together with the simplex constraint $\sum_i \hat\tp^{\star}_i = 1$ forces $\hat\tp^{\star} = e_{i^{\star}} = \tp^{\star}$ exactly: the vertex is determined by its identity, independent of the numerical values of $\hat\tA$.
Consequently $\hat p_{\F} = p^{\star}_{\F}$ exactly on $\mathcal E$, and per-round commit regret is zero.

Exploration contributes at most $\tau m\ell$ rounds of regret (each $\le 1$); the failure event $\mathcal E^{c}$ (probability $\le 1/T$, regret bounded by $T$) contributes at most $1$.
Therefore $R_T^\fair \le \tau m\ell + 1$, which on substituting $\tau$ gives the displayed bound.
\end{proof}

\begin{corollary}[Recovery of the unconstrained pure-NE regime]
\label{cor:pure_recovery}
At $\alpha=0$, $\tA=A$ and Assumption~\ref{ass:dominant} reduces to the existence of a pure NE; the $2\times2$ basis gives $\kappa(0)\le 2$ (Remark~\ref{rem:scope}), and $\tDelta(0)$ is the minimum LP margin of the unconstrained NE. Theorem~\ref{thm:fetc_inst} then yields $R_T^\fair\le O(m\ell\log(m\ell T)/\tDelta(0)^2)$, an instance-dependent $O(\log T)$ rate, recovering the same regime as the pure-NE analysis of \citet{yilmaz2025}.
\end{corollary}

\begin{remark}[Scope and validity range of Theorem~\ref{thm:fetc_inst}]
\label{rem:scope}
The range of $\alpha$ over which Theorem~\ref{thm:fetc_inst} is useful is bounded away from $1$, for two reasons. \emph{Structurally}, Assumption~\ref{ass:dominant} holds only on a sub-interval $[0,\tilde\alpha_{\mathrm{struct}})$, non-trivial whenever the unconstrained game has a strict pure NE. \emph{Quantitatively}, even inside that interval the gap shrinks as $\tDelta(\alpha)\le(1-\alpha)\delta_{\mathrm{row}}$, so the leading term grows at least as $(1-\alpha)^{-2}$ and the bound is non-vacuous only for $\alpha\le\tilde\alpha(T)$, with $\tilde\alpha(T)\uparrow1$ as $T\to\infty$. The blowup is carried by $\tDelta(\alpha)$, not the conditioning: in the $s=1$ regime $\kappa(\alpha)\le2$ uniformly (cf.\ Corollary~\ref{cor:pure_recovery}). Appendix~\ref{app:scope} gives the precise structural and margin conditions.
\end{remark}

\subsection{Why naive elimination fails}
\label{sec:elim}

A natural question is whether adaptive action elimination can sharpen the $T^{2/3}$ rate by removing actions from consideration once they appear suboptimal.
The obstruction is structural: under fairness, an action's value at the fair NE is determined by its row of $\tA$ in conjunction with the equilibrium column distribution $q^{\star}_{\F}$, not by its standalone column-by-column comparisons against other rows.
Note that the rank-one shift $A\mapsto\tA$ adds a column-constant $\alpha c_{j}$ and scales by $(1-\alpha)>0$; both operations preserve the within-column ordering of rows, so a row dominated column-wise in $A$ remains dominated in $\tA$ and the shift never promotes it.
The genuine difficulty is that equilibrium \emph{support} need not track any per-action ranking: a row with low average (or low column-wise) payoff can carry the largest equilibrium mass, because that mass is fixed by $q^{\star}_{\F}$, an equilibrium quantity unknown to the learner before the LP is solved.

\begin{example}[Average-payoff ranking misidentifies the active row]
\label{ex:elim}
Let $m=\ell=2$ and $A=\bigl(\begin{smallmatrix}0.9 & 0.1\\ 0.3 & 0.6\end{smallmatrix}\bigr)$ (take $\alpha=0$, so $\tA=A$).
The game has no pure NE; its unique equilibrium is mixed with $p^{\star}\approx(0.27,\,0.73)$, so row $2$ carries the larger mass. Yet row $2$ has the \emph{smaller} row average ($0.45$ versus $0.50$), so an elimination rule that demotes the lower-average (or lower-min) arm would discard precisely the dominant equilibrium action. Passing to $\tA$ only compresses payoff gaps toward column means, blurring such per-action rankings further.
\end{example}

A correct demotion criterion must instead respect LP structure, demoting only actions whose \emph{row gap} $\rho_i(\alpha)$ (Definition~\ref{def:gaps}) is certifiably positive at a confidence level matching the estimation noise and the conditioning factor $\kappa(\alpha)$.
Any such test requires estimating $\tA$ and $q^{\star}_{\F}$ jointly, and so cannot be reduced to per-action or per-column statistics in the way classical bandit elimination is.
We defer the design and analysis of LP-aware elimination algorithms to future work.

\begin{remark}[Fairness--exploration synergy]
\label{rem:synergy}
A useful side benefit of fairness: under \alg, every row action is played with probability $\ge\alpha/m$ in every round of the commit phase, so all rows continue to be sampled passively even after commitment.
In an anytime variant that interleaves exploration and exploitation, this provides a built-in self-correction mechanism absent in unconstrained ETC.
\end{remark}

\section{Related work}
\label{sec:related}

\paragraph{Learning in zero-sum games with bandit feedback.}
\citet{ito2025} prove $O(c_1\log T + \sqrt{c_2 T})$ for self-play of Tsallis-INF \citep{zimmert2021} on normal-form games.
\citet{yilmaz2025} analyze ETC variants for TPZSGs with pure NE.
\citet{maiti2023} establish sample-complexity bounds for identifying pure-strategy Nash equilibria.
\citet{wei2021} address last-iterate convergence in constrained saddle-point optimization with full-information feedback.
All of these focus on either pure NE or unconstrained mixed NE; we study the constrained-mixed regime where fairness fundamentally changes the equilibrium structure.

\paragraph{Fairness in bandits and reinforcement learning.}
\citet{joseph2016} introduce a meritocratic fairness notion for stochastic and contextual bandits.
\citet{liu2017} study calibrated fairness via Thompson sampling.
\citet{patil2021} treat minimum-pull constraints in single-agent stochastic bandits (no opponent) and obtain $O(\log T)$ $r$-regret with UCB1. \citet{jabbari2017} extend fairness to reinforcement learning, and \citet{siddique2020,weng2019} learn fair policies via generalized Gini welfare \citep{weymark1981} in multi-objective settings.
Our work bridges fairness-constrained learning with adversarial game-theoretic settings under partial feedback.

\paragraph{Constrained MDPs and online learning under constraints.}
The interaction of online learning with hard constraints has been studied extensively in CMDPs \citep{altman1999,efroni2020} and online convex optimization with constraints \citep{mahdavi2012}.
Our reparametrization $\tA = (1-\alpha)A + \alpha\ind c^\top$ is in the spirit of barrier and proximal-point reformulations of constrained convex optimization \citep{rockafellar1970}; the specific use to embed minimum-play fairness into a standard zero-sum game appears to be novel.

\section{Conclusion and future directions}
\label{sec:conclusion}

We introduced fair equilibria in TPZSGs with bandit feedback, reduced the fair game to a standard zero-sum game on a \emph{fair payoff matrix} $\tA$ via affine reparametrization of $\F_\alpha$, proposed \alg, and established an $O(T^{2/3})$ regret bound for general mixed fair NE together with an $O(\log T)$ instance-dependent rate when $\tp^\star$ is a vertex of $\Delta_m$.
Several directions remain open.

An adaptive UCB or Tsallis-INF-style algorithm \citep{zimmert2021,ito2025} projected onto $\F_\alpha$ that continuously re-solves the fair LP on improving estimates could plausibly achieve $\widetilde O(\sqrt T)$ for mixed fair NE; the key technical step is a stability bound for the NE LP under simultaneous perturbations, accounting for the rank-one structure of $\tA$.

Section~\ref{sec:elim} argues that elimination demands LP-aware criteria. An algorithm demoting actions whose row gap $\rho_i(\alpha)$ is certifiably positive (with confidence calibrated by $\kappa(\alpha)$) may achieve instance-dependent rates for general fair NE; analyzing the propagation of estimation error through the LP basis is the key technical challenge.

Extending to two-sided constraints (Remark~\ref{rem:twosided}) or replacing minimum-play with generalized Gini welfare functions \citep{weymark1981,siddique2020} are natural next steps and would connect our framework to the broader fair-RL literature.

\section*{Acknowledgments}

The work of Pratik Gajane was supported by the French National Research Agency (ANR) under grant ANR-24-CPJ1-0088-01. The funders had no role in the study design, analysis, the conclusions drawn, or the writing of this manuscript.

\bibliographystyle{plainnat}
{\small\bibliography{references}}

\appendix

\section{Auxiliary calculations for Theorem~\ref{thm:fetc_general}}
\label{app:proof}

From \eqref{eq:regret_raw}, $R_T^\fair \le \tau m\ell + 2T\sigma\sqrt{2L/\tau} + 1$ with $L = \log(2m\ell T)$.
Let $a := m\ell$, $b := T\sigma\sqrt{2L}$, and $f(\tau) := a\tau + 2b\tau^{-1/2}$ on $\tau > 0$.

\paragraph{Critical point.}
$f'(\tau) = a - b\tau^{-3/2}$, vanishing at $\tau^{\star} = (b/a)^{2/3}$.
$f''(\tau) = (3/2)b\tau^{-5/2} > 0$, so $\tau^{\star}$ is the unique global minimum.

\paragraph{Value at the minimum.}
$a\tau^{\star} = a^{1/3}b^{2/3}$ and $2b(\tau^{\star})^{-1/2} = 2a^{1/3}b^{2/3}$, so $f(\tau^{\star}) = 3a^{1/3}b^{2/3} = 3(ab^2)^{1/3} = 3(2\sigma^2 m\ell T^2 L)^{1/3}$.

\paragraph{Integer rounding.}
Choosing $\tau = \lceil\tau^{\star}\rceil \in [\tau^{\star}, \tau^{\star}+1]$, the linear part $a\tau$ increases by at most $a = m\ell$; the $\tau^{-1/2}$ part decreases.
Thus $f(\lceil\tau^{\star}\rceil) \le f(\tau^{\star}) + m\ell$.

\paragraph{Horizon admissibility.}
We require $\tau m\ell \le T$, i.e., $\tau^{\star} m\ell \le T$, equivalently $T \ge 2\sigma^2 m\ell L = \Omega(m\ell\log(m\ell))$ for $L = \log(2m\ell T)$.

Combining yields $R_T^\fair \le 3(2\sigma^2 m\ell T^2 L)^{1/3} + m\ell + 1$.

\section{Non-degeneracy of the NE LP on $\tA$}
\label{app:nondeg}

We argue that the strict-complementarity hypothesis in Lemma~\ref{lem:basis} is generic in $A$.
The set of payoff matrices $A\in\R^{m\times\ell}$ for which the standard NE LP on $\tA = (1-\alpha)A + \alpha\ind c^\top$ has a degenerate or weakly-complementary optimum is a finite union of algebraic varieties of strictly lower dimension than $m\ell$, hence of Lebesgue measure zero in $\R^{m\times\ell}$, for any fixed $\alpha$. (Genericity is in $A$ at fixed $\alpha$; it does not assert non-degeneracy at every $\alpha$ for a fixed $A$, indeed, the line $\alpha\mapsto\tA$ can cross the degeneracy variety at isolated $\alpha$, precisely where Assumption~\ref{ass:dominant} expires, cf.\ Remark~\ref{rem:scope}.)

This follows because degeneracy of an LP corresponds to the simultaneous satisfaction of more than $\dim(\mathcal B)$ equality constraints, and weak complementarity to the vanishing of a single reduced cost; each such condition is a polynomial equation in the entries of $A$, defining a proper subvariety \citep[generic LP non-degeneracy, see e.g.,][Ch.~6]{schrijver1986}.
The map $A\mapsto\tA$ is affine with constant Jacobian and so preserves Lebesgue-null sets.
Consequently, all results in Section~\ref{sec:fetc} hold with probability one over Lebesgue-random $A$, and for payoff matrices on the boundary of the degeneracy varieties a vanishingly small perturbation restores non-degeneracy.

\section{Proof of Lemma~\ref{lem:basis}}
\label{app:basis}

This is a quantitative LP-sensitivity argument \citep[Ch.~5]{bertsimas1997}.
Basis $\mathcal B$ remains optimal on $\hat\tA$ iff (i) the basic primal solution remains feasible, and (ii) the dual reduced costs remain non-negative.

\emph{(i) Primal sensitivity.}
Recall from Definition~\ref{def:primal} that $\mathcal B$ consists of the simplex equality $\sum_i \tp_i = 1$, the non-negativity equalities $\tp_i = 0$ for $i\in\Sflr$, and the column inequalities $\sum_i \tp_i \tA(i,j) = \nu$ for $j\in\supp(q^{\star}_{\F})$.
The basic variables are $z := (\tp^{\star}_{\Sact},\,\nu) \in \R^{s+1}$, and they satisfy the square linear system $M_{\mathcal B}\,z = b_{\mathcal B}$.
Here $M_{\mathcal B}$ is built by, for each active row of $\mathcal B$, restricting its coefficient vector to the basic variables: the simplex row contributes $(\ind^{\top}_{\Sact},\,0)$ with right-hand side $1$; each tight column constraint $j\in\supp(q^{\star}_{\F})$ contributes the row $(\tA(\cdot,j)|_{\Sact}^{\top},\,-1)$ with right-hand side $0$.
By non-degeneracy and strict complementarity at the optimal vertex, $|\supp(q^{\star}_{\F})|=s$, so $M_{\mathcal B}$ is a square invertible matrix of dimension $s+1$, and the basis (hence $M_{\mathcal B}$, $\kappa(\alpha)$) is unique.
The non-negativity equalities for $i\in\Sflr$ are automatic since those variables are non-basic and set to $0$; they contribute no rows to $M_{\mathcal B}$.

Under $\tA\to\hat\tA$, only the column-constraint rows change, and only in their basic-variable entries $\tA(\cdot,j)|_{\Sact}$.
Each such row has $s$ perturbed entries shifting by at most $\eps$, while the $-1$ entry for $\nu$ is unchanged; the simplex row is unchanged.
Since the $\ell_{\infty}$-induced norm equals the maximum absolute row sum, each row-sum changes by at most $s\eps$, so $\|M_{\mathcal B}-\hat M_{\mathcal B}\|_{\infty}\le s\eps$.
The hypothesis $s\kappa(\alpha)\eps \le 1/2$ then gives
\[
\|M_{\mathcal B}^{-1}(M_{\mathcal B}-\hat M_{\mathcal B})\|_{\infty} \;\le\; \kappa(\alpha)\,\|M_{\mathcal B}-\hat M_{\mathcal B}\|_{\infty} \;\le\; s\kappa(\alpha)\eps \;\le\; \tfrac{1}{2},
\]
and the Neumann series gives $\hat M_{\mathcal B}$ invertible with $\|\hat M_{\mathcal B}^{-1}\|_{\infty}\le 2\kappa(\alpha)$.
From $\hat M_{\mathcal B}\hat z = b_{\mathcal B}$ and $M_{\mathcal B}z = b_{\mathcal B}$, subtracting gives $\hat M_{\mathcal B}(\hat z - z) = (M_{\mathcal B}-\hat M_{\mathcal B})z$, hence
\[
\|\hat z - z\|_{\infty} \;\le\; \|\hat M_{\mathcal B}^{-1}\|_{\infty}\,\|M_{\mathcal B}-\hat M_{\mathcal B}\|_{\infty}\,\|z\|_{\infty} \;\le\; 2\kappa(\alpha)\cdot s\eps\cdot 1 \;=\; 2s\kappa(\alpha)\,\eps,
\]
using $\|z\|_{\infty}\le 1$ since $\tp^{\star}_{\Sact}\in[0,1]^{s}$ and $\nu\in[0,1]$.
This yields $\|\hat\tp^{\star}-\tp^{\star}\|_{\infty}\le 2s\kappa(\alpha)\eps$ and $|\hat\nu-\nu|\le 2s\kappa(\alpha)\eps$.
Because $\eps\le\tDelta(\alpha)/(C\kappa(\alpha))$ with $\tDelta(\alpha)\le\pi(\alpha)$ and $C=2s^2+2s+1>2s$, we get $2s\kappa(\alpha)\eps \le 2s\tDelta(\alpha)/C \le 2s\pi(\alpha)/C < \pi(\alpha)$, so $\hat\tp^{\star}_i>0$ for all $i\in\Sact$ and primal feasibility on $\mathcal B$ is preserved.

\emph{(ii) Dual sensitivity.}
For a non-basic column $j\notin\supp(q^{\star}_{\F})$, the perturbed column slack is
\[
\hat\gamma_j \;=\; \textstyle\sum_i \hat\tp^{\star}_i\,\hat\tA(i,j) - \hat\nu,
\]
and
\[
\hat\gamma_j - \gamma_j \;=\; \underbrace{(\hat\tp^{\star}-\tp^{\star})^{\top}\hat\tA_{\cdot,j}}_{\text{(a)}} \;+\; \underbrace{(\tp^{\star})^{\top}(\hat\tA_{\cdot,j}-\tA_{\cdot,j})}_{\text{(b)}} \;-\; \underbrace{(\hat\nu - \nu)}_{\text{(c)}}.
\]
Since $\hat\tp^{\star}-\tp^{\star}$ is supported on $\Sact$, $\|\hat\tp^{\star}-\tp^{\star}\|_1 \le s\|\hat\tp^{\star}-\tp^{\star}\|_{\infty} \le 2s^{2}\kappa(\alpha)\eps$; with $\|\hat\tA_{\cdot,j}\|_{\infty}\le 1$ this gives $|\text{(a)}|\le 2s^{2}\kappa(\alpha)\eps$.
Using $\|\tp^{\star}\|_1 = 1$ and $\|\hat\tA_{\cdot,j}-\tA_{\cdot,j}\|_{\infty}\le\eps$: $|\text{(b)}|\le \eps$.
From (i): $|\text{(c)}|\le 2s\kappa(\alpha)\eps$.
Therefore, using $\kappa(\alpha)\ge 1$ to bound $\eps\le\kappa(\alpha)\eps$,
\[
|\hat\gamma_j - \gamma_j| \;\le\; 2s^{2}\kappa(\alpha)\eps + \eps + 2s\kappa(\alpha)\eps \;\le\; (2s^{2}+2s+1)\kappa(\alpha)\eps \;=\; C\kappa(\alpha)\eps,
\]
where $\kappa(\alpha)\ge 1$ holds for any basis containing the simplex-equality row $\ind^{\top}$.
A symmetric bound holds for the row slacks $\hat\rho_i$, $i\in\Sflr$.
With $\eps\le\tDelta(\alpha)/(C\kappa(\alpha))$ we get $|\hat\gamma_j-\gamma_j|\le\tDelta(\alpha)$, so for every non-basic $j$ with $\gamma_j>0$ (hence $\gamma_j\ge\tDelta(\alpha)$) the perturbed slack satisfies $\hat\gamma_j\ge\gamma_j-\tDelta(\alpha)\ge 0$, and likewise $\hat\rho_i\ge 0$ for $i\in\Sflr$.
All non-basic reduced costs thus remain non-negative on $\hat\tA$, so $\mathcal B$ stays optimal; combined with the strict primal feasibility from (i) and the non-degeneracy hypothesis, $\mathcal B$ remains the unique optimal basis and the active set is preserved.

Combining (i) and (ii) with $C := 2s^{2}+2s+1$ completes the proof; bounding $s\le m$ gives $C\le 2m^{2}+2m+1$ uniformly, and at $s=1$ (Theorem~\ref{thm:fetc_inst}) $C=5$.
Since the basis determines $(\Sact,\Sflr,\supp(q^{\star}_{\F}))$, the conclusion follows.
\qed

\section{Scope and validity range of Theorem~\ref{thm:fetc_inst}}
\label{app:scope}

Two $\alpha$-dependent effects bound the range over which Theorem~\ref{thm:fetc_inst} is useful.

\emph{Structure.} Assumption~\ref{ass:dominant} holds not for all $\alpha\in[0,1)$ but on a sub-interval $[0,\tilde\alpha_{\mathrm{struct}})$ starting at $0$, requiring the unconstrained pure NE to have gap large enough to survive the rank-one shift $A\to\tA$. By continuity this interval is non-trivial whenever the unconstrained game has a strict pure NE, so the regime is not confined to $\alpha=0$. When the column means $c_j$ are distinct, the limiting active column is $j^*=\argmin_j c_j$ and the structure survives all of $[0,1)$ iff the dominant action retains a strictly positive advantage $\min_{i\neq i^\star}[A(i^\star,j^*)-A(i,j^*)]>0$ in that column. The boundary case of constant $c_j$ is even simpler: the shift becomes a global additive constant plus a positive rescaling, leaving the equilibrium $\alpha$-invariant, so the structure survives all of $[0,1)$ automatically (as in Example~\ref{ex:nondeg}, where $c=(0.5,0.5)^\top$ and $q^\star_\F=e_2$ for every $\alpha$).

\emph{Margin.} Even within $[0,\tilde\alpha_{\mathrm{struct}})$ the bound degrades \emph{continuously} as $\alpha\to1$, since $\tA\to\ind c^\top$ (identical rows) in the limit. By the dual-gap identity of Section~\ref{sec:gaps}, with $\pi(\alpha)=1$ under Assumption~\ref{ass:dominant} the primal margin never binds, so
\[
\tDelta(\alpha)\;\le\;\min_{i\neq i^\star}\rho_i(\alpha)\;=\;(1-\alpha)\,\delta_{\mathrm{row}}(\alpha),
\qquad
\delta_{\mathrm{row}}(\alpha):=\min_{i\neq i^\star}\big[(Aq^\star_\F)_{i^\star}-(Aq^\star_\F)_i\big].
\]
Hence $\tDelta(\alpha)\to0$ at least linearly in $(1-\alpha)$, and the leading term of Theorem~\ref{thm:fetc_inst} grows at least as $(1-\alpha)^{-2}$. The blowup is carried by $\tDelta(\alpha)$, not $\kappa(\alpha)$: in the $s=1$ regime $M_{\mathcal B}$ is $2\times2$ with entries in $[0,1]$, so $\kappa(\alpha)\le2$ uniformly (cf.\ Corollary~\ref{cor:pure_recovery}); the $\kappa\to\infty$ mechanism arises only for $s\ge2$, outside this theorem.

\emph{Consequence.} Since the bound is non-vacuous (below $T$) only when $\tDelta(\alpha)\gtrsim\sigma\sqrt{m\ell\log(m\ell T)/T}$, the upper estimate on $\tDelta$ gives a \emph{necessary} condition: the usable range is no larger than $[0,\tilde\alpha(T))$ with $\tilde\alpha(T):=1-\Theta\big(\tfrac{\sigma}{\delta_{\mathrm{row}}(1^-)}\sqrt{m\ell\log(m\ell T)/T}\big)$. This widens with the horizon ($\tilde\alpha(T)\uparrow1$ as $T\to\infty$) and with the game's gap (larger $\delta_{\mathrm{row}}(1^-)$ tolerates more fairness), and collapses as $\delta_{\mathrm{row}}(1^-)\to0$. For a fixed horizon the usable range is bounded away from $1$, capped by $\min(\tilde\alpha_{\mathrm{struct}},\tilde\alpha(T))$. The endpoint $\alpha=1$ is excluded regardless, where the floor forces the uniform strategy and there is nothing to learn.

\section{Compatibility with the original constrained LP formulation}
\label{app:lp_compat}

The fair LP can be stated directly on $A$ without invoking $\tA$:
\begin{equation}
\label{eq:lp_app}
\max_{p\in\R^m,\,v\in\R}\; v \quad\text{s.t.}\quad \sum_i p_i A(i,j)\ge v\;\forall\,j,\quad p_i\ge\alpha/m\;\forall\,i,\quad \sum_i p_i = 1.
\end{equation}
By Proposition~\ref{prop:reduction}, \eqref{eq:lp_app} is equivalent (via the affine bijection $p\leftrightarrow\tp$) to the standard NE LP on $\tA$.
Solving \eqref{eq:lp_app} on $\hat A$ returns the same row strategy as solving the standard NE LP on $\hat\tA$.
The two formulations are computationally interchangeable; we phrase Algorithm~\ref{alg:fetc} and the analysis via $\tA$ because basis-perturbation theory (Lemma~\ref{lem:basis}) is cleaner without floor constraints, which would introduce additional Lagrange multipliers and a separate primal-feasibility analysis.

\end{document}